\documentclass[11pt]{article}
\usepackage{fullpage}
\usepackage{times}
\usepackage{soul}
\usepackage{url}
\usepackage[hidelinks]{hyperref}
\usepackage[utf8]{inputenc}
\usepackage[small]{caption}
\usepackage{graphicx}
\usepackage{amsmath}
\usepackage{amsthm,amssymb,mathtools}
\usepackage{booktabs}
\usepackage{algorithmic}
\usepackage[ruled,vlined,linesnumbered]{algorithm2e}
\usepackage{textcomp}
\usepackage{xcolor}

\usepackage{subcaption}
\usepackage{pgfplots}
\usepgfplotslibrary{external}
\usepackage{tikz}

\newtheorem{definition}{Definition}
\newtheorem{theorem}{Theorem}
\newtheorem{lemma}{Lemma}

\newcommand{\brc}{{\sc brc}\ }
\usepackage{hyperref}
\pgfplotsset{width=10cm,compat=1.9}

\title{Classification with Partially Private Features}
\author{Zeyu Shen\thanks{Duke University, Durham NC 27708. Email: {\tt zeyu.shen@duke.edu}} \and Anilesh K. Krishnaswamy\thanks{Google, Mountain View, CA.} \and Janardhan Kulkarni\thanks{Microsoft Research, Redmond, WA.} \and Kamesh Munagala\thanks{Department of Computer Science, Duke University, Durham NC 27708-0129. Supported by NSF grant  CCF-2113798. Email: {\tt kamesh@cs.duke.edu}}}
\date{}
    
\begin{document}

\maketitle

\begin{abstract}
    In this paper, we consider differentially private classification when some features are sensitive, while the rest of the features and the label are not. We adapt the definition of differential privacy naturally to this setting. Our main contribution is a novel adaptation of AdaBoost that is not only provably differentially private, but also significantly outperforms a natural benchmark that assumes the entire data of the individual is sensitive in the experiments. As a surprising observation, we show that boosting randomly generated classifiers suffices to achieve high accuracy. Our approach easily adapts to the classical setting where all the features are sensitive, providing an alternate algorithm for differentially private linear classification with a much simpler privacy proof and comparable or higher accuracy than differentially private logistic regression on real-world datasets.
    
\end{abstract}

\section{Introduction}
Privacy of data has become increasingly important in large-scale machine learning applications, where data points correspond to individuals that seek privacy. Classifiers are often trained over data of individuals with sensitive attributes about individuals: income, education, marital status, and so on for instance. A well-accepted way to incorporate privacy into machine learning is the framework of {\em differential privacy}~\cite{calibrating_noise,Dwork2011AFF,algo_diff_privacy}. The key idea is to add noise either to individual data items or to the output of the classifier so that the distribution of classifiers produced is mathematically close when an arbitrary individual is added or removed from the dataset. This provides a quantifiable way in which sensitive data about any individual is information theoretically secure during the classification process. All this comes at a cost: Adding noise leads to a loss in accuracy of the classifier, and as we elaborate below, a large body of work has studied the privacy-accuracy trade-off both theoretically and empirically.

\paragraph{Classification with Partially Private Features.} Many classification algorithms work better as the number of features or attributes of the individual that are provided as input to it increases. Indeed, modern deep learning methods are notorious for being ``feature hungry'' in that sense. In this paper, we consider the scenario where some of these features are private, and the rest, including the label, are public. 

We first motivate this model via several examples. For instance, suppose the classification task is predicting whether someone will use ride-share services a lot in a year. Such a prediction task is ``low-stakes'' and an individual will likely not care about the privacy of the label itself. Further, most counties maintain property or voting records that can be used to infer address, race, and gender, and these can be viewed as public attributes. The classification could also use some sensitive features such as income level or source/destination of trips, and these are the features on which we wish to enforce privacy. Similarly, a security firm may want to classify users who update their computers more frequently from those who do not. Again, the label itself is likely not sensitive, and the classification might use some sensitive features like web browsing patterns, but could be based on largely public information like job description, location of primary use, etc. Another example is image classification, for instance, predicting whether an image (where part of the image is private) is a dog; again here, the label itself is not sensitive compared to other parts of the image. 

What is more, this model is relevant even in some ``high stakes'' decision making settings. As one example, in predicting recidivism, court records of past arrests (label) are typically public information. As another example, consider constructing a predictor of individual factors affecting college admissions or winning of prestigious awards. Here, the outcome of what college the student attended, or who won the awards is typically public information, while features like family income or race can be sensitive. Similarly, consider a university trying to determine personal factors (collected via, say an individual survey) contributing to research success measured by some publicly known outcome such as citation metrics. In this case, the survey information is confidential, while the label is not. Both these examples are ``high stakes'', since the outcome of analysis could be used by policy makers or courts to determine the fairness of processes that evaluate the student or researcher.  

In these scenarios, it makes sense to categorize the features themselves as being either ``public'' or ``private'', so that the privacy guarantee of the classifier applies only to private features, and not public features. The hope is that for several classification scenarios, if some features are public, the amount of noise required to achieve privacy is lower, and we can hope to achieve higher accuracy. Indeed, in the sequel, we show that this way of viewing the problem leads to an algorithmic framework that also provides better empirical results for the case where {\em all} features and label are private.


\paragraph{Differentially Private Boosting.} Though it seems intuitively clear that having public features and label should help, existing algorithms for differentially private classification either add noise to an individual, or to the gradient of the loss function, or the output classifier. For all these approaches, it is unclear how having public features and label reduces the noise needed for privacy. In contrast, a natural framework for classification with two types of features is {\em boosting}, particularly the AdaBoost algorithm~\cite{adaboost}. (Please see Algorithm~\ref{algo0} in Section~\ref{sec:prelim} for pseudo-code of AdaBoost.) At each step, 
we generate two classifiers: one for public features and the other for private features. Based on the accuracy of these classifiers, the weights of the data points are updated via exponentiation and the process run for several iterations. This overall framework of combining AdaBoost with differential privacy first appeared in~\cite{boosting-dp}. These authors present a formal guarantee on the accuracy-privacy trade-off achievable via boosting.

While current algorithms are asymptotically optimal, they do not perform well in practice. Indeed, as we show in Section~\ref{sec:discuss}, the method in~\cite{boosting-dp} yields low accuracy for reasonable settings of the privacy parameter. The key reason is that two steps in AdaBoost need to be simultaneously made private. The first is the step that calculates the accuracy of the private classifier each step, which determines how the weights of the individuals are updated. The second is the weights themselves that are passed to the subroutine that does private classification given these weights.
Adding noise at both places is required for privacy, but as we show, on real datasets, such an algorithm performs no better than randomly guessing the labels for reasonable settings of the privacy parameters. Indeed, its accuracy is much worse than simply running a differentially private classifier  assuming all features are private.

This brings up the main question we study in this paper:
\begin{quote}
    How should a DP classifier be designed when only some features of the individual are private, so that its accuracy improves over classifiers that require all features as well as label  to be private?
\end{quote}

\subsection{Our Contributions}
In this paper, we consider {\em binary linear} classification problems where logistic regression empirically performs well, and where the natural algorithm for differential privacy is differentially private logistic regression~\cite{logistic}. We consider the setting where the features of the individual are split between publicly available and private. As mentioned before, we assume the label is also public information. We present a framework for this problem that beats natural baselines on several data-sets, while being  differentially private. An advantage of this framework is that it also naturally extends to the setting where all features and labels are private, thereby providing a private and empirically viable alternative to differentially private logistic regression. 

We first consider the setting with partially private features. Our main contribution is a novel modification of AdaBoost that we present as Algorithm~\ref{algo1} in Section~\ref{sec:alg}. 
At a high level, our algorithm works as follows: In each step of AdaBoost, we learn two classifiers - a randomly generated linear classifier for private features, which we call a private classifier, and a classifier for public features that is learnt using weights on the data points, which we call a public classifier. We pick whichever classifier whose accuracy is farther away from random guessing (50 \% accuracy). To emphasize, weights on the data points are not used in generating private classifiers. They are only used to estimate the accuracy of the classifiers and train public classifiers.
We term this method {\em Boosting with Random Classifiers} or {\sc brc}. The advantage of \brc is that we only need to make the accuracy computation of the classifiers private, but not the weights of the data points. In other words, we do not need to add noise on the weights of the data points. We conjecture that the method in~\cite{boosting-dp} do not work well in practice because they add too much noise on the weights of the data points, and our algorithm gets rid of this. We present a detailed comparison with this approach in Section~\ref{sec:compare} and perform a privacy analysis of \brc in Section~\ref{sec:analysis}. 


In Section~\ref{sec:toy}, we first try to demystify our algorithm. In particular, we ask the question: For the subroutine that finds the private classifier (the random linear classifier) in each iteration, why should ignoring the data points even work? Our explanation is simple: AdaBoost will perform reasonably well even if the classifiers in each iteration are completely random, provided each classifier has slightly better accuracy than one half. We verify this intuition on a one-dimensional toy example with a perfect classifier. We also shed light on the parameter setting using this example.

We implement \brc for binary classification tasks with a split between public and private features, and perform an empirical study in Section~\ref{sec:experiments}. We show that on several datasets with natural splits of features into public and private, and several natural choices of the privacy parameter, we show that \brc is a significant improvement over two natural baselines: running logistic regression only on public features (with no noise), and running differentially private logistic regression treating all features as private.

In Section~\ref{sec:allpriv}, we consider the classical setting  of differentially private linear classification where all features and labels are private. We show that the \brc adapts to this setting straightforwardly. This leads to not only a simpler proof of privacy compared to differentially private logistic regression~\cite{logistic}, but also to comparable or higher accuracy on real-world data. 

In Section~\ref{sec:discuss}, we explore our algorithmic framework in a bit more depth. We first compare \brc with the Boosting with People framework from~\cite{boosting-dp}, and show that the latter approach does no better than random guessing under the same settings of privacy parameters that \brc uses. This motivates the need for the \brc framework as a practical solution for classification with partially private features. Moreover, since the Private Aggregation of Teacher Ensembles (PATE) framework from~\cite{pate1, pate2, tran2022sfpate} stands out as an effective framework that deals with scenarios with both public and private information, we compare \brc with an extension of PATE to our setting even though PATE was originally designed for settings fundamentally different from ours. We show that the accuracy of PATE increases very slowly with privacy budgets, and it is superseded by \brc for most of the privacy parameters we consider. We also show that \brc has similar rate of convergence as non-private AdaBoost. Finally, since all datasets that we choose in Section~\ref{sec:experiments} have large size, we complement our results by showing that \brc performs similarly well on small datasets.


At a higher level, AdaBoost can somewhat magically generalize  even with weak classifiers; see~\cite{boost-general} and citations within for a thorough discussion. Our work in effect exploits this phenomenon in an extreme sense by simply choosing random classifiers each iteration. As far as we are aware, such a use of AdaBost is novel, and we hope it finds other applications in private learning or otherwise.  


\subsection{Related Work}
There is a large body of work on extending well-known machine learning classifiers to make them differentially private. For example, the work of~\cite{logistic} investigates empirical risk minimization with differential privacy. They present a method for privacy-preserving logistic regression and support vector machines, where the key idea here is to add carefully chosen noise to the objective function.  The work of~\cite{random_forest,rf2} introduces differentially private random forests, while the work of~\cite{bayes,linear} extends linear regression and naive Bayes to differential privacy. These algorithms have been implemented in libraries~\cite{diffprivlib} and perform well for simple classification tasks. We use these implementations in our experiments. However, the performance of these algorithms often degrade when the classification task becomes more complex: in image classification or with multiple classes for instance. To overcome this problem, the work of~\cite{dpsgd} introduces deep learning with differential privacy via making stochastic gradient descent  itself  private. This and subsequent work establishes state-of-the-art privacy-accuracy trade-off on benchmark multi-class classification datasets. 

The Private Aggregation of Teacher Ensembles (PATE) framework~\cite{pate1,pate2, tran2022sfpate} is in the same spirit as ours: Instead of considering the case where some features can be public, they assume some of the training data can be publicly available. They split the private data into batches and privately aggregate the ``teacher'' classifier results for these batches, feeding it as a feature into a ``student'' model trained using publicly available data. 


Following the work of~\cite{boosting-dp}, many works have recently incorporated boosting into differentially private classification to achieve better privacy-accuracy trade-off. For instance, the work of~\cite{Li2020PrivacyPreservingGB,gbdt2,gbdt3,gbdt1} studies and improves the effectiveness of private Gradient Boosted Decision Trees, that combine decision trees with boosting. Similarly, the work of~\cite{Nori2021AccuracyIA} shows that adding differential privacy to Explainable Boosting Machines, a recent method for training interpretable machine learning models, yields state-of-the-art accuracy while preserving privacy. 

The above work assumes all attributes are private, and it is not at all clear how to effectively extend these algorithms to the setting where there is a mix of public and private features and obtain improved accuracy-privacy trade-off. Our work attempts to fill this gap, and our framework can be used with any of the above differentially private classifiers run on private features as a subroutine.  

\section{Preliminaries: Differential Privacy and AdaBoost}
\label{sec:prelim}
We first review the classic AdaBoost algorithm in Algorithm~\ref{algo0}. We assume a classification task is specified by set of observations  $\{x_i\}_{i = 1}^n$ and corresponding labels $\{y_i\}_{i = 1}^n$. These labels can either be binary valued in binary classification, or take on multiple discrete values in multiclass classification. In this paper, we mainly focus on binary classification. The AdaBoost~\cite{adaboost} algorithm assumes the existence of a ``weak learner'', one whose performance, given any weights on the data points, is slightly better than random guessing. It creates an ensemble of such classifiers by iteratively re-weighting points,  in that process finding a strong classifier whose training loss approaches zero. The pseudocode for AdaBoost can be found in Algorithm~\ref{algo0}.

\begin{algorithm}[!htbp]
\SetAlgoLined
{\bf Input:} Number of iterations $T$, observations $\{x_i\}_{i = 1}^n$ and corresponding labels $\{y_i\}_{i = 1}^n$.\\
{\bf Initialization:} Observation weights $w_i = 1\quad \forall i \in \{1, \ldots, n\}$.

\For{$t = 1, \ldots, T$}{
Fit a classifier $h_t(x)$ to the training data using weights $\{w_i\}_{i = 1}^n$\;
Compute $\text{err}_t = \frac{\sum_{i = 1}^nw_i\mathbf{1}(y_i \neq h_t(x_i))}{\sum_{i = 1}^n w_i}$\;
Compute $\alpha_t = \log \left(\frac{1 - \text{err}_t}{\text{err}_t}\right)$\;
Set $w_i = w_i\exp(\alpha_t \mathbf{1}(y_i \neq h_t(x_i)))\quad \forall i \in \{1, \ldots, n\}$\;
}
{\bf Output:} $H(x) = \text{sign}\left(\sum_{t = 1}^T\alpha_th_t(x)\right)$.
 \caption{The AdaBoost Algorithm}
 \nllabel{algo0}
\end{algorithm}

We next discuss differential privacy. Consider the universe of all datasets $\mathcal{D}$. For two datasets $D, D' \in \mathcal{D}$, we use $\Vert D - D'\Vert$ to denote the number of data points {\em whose private (or sensitive) features} differ between $D$ and $D'$. We assume the label is public knowledge. Differential privacy \cite{Dwork2011AFF,calibrating_noise,algo_diff_privacy} is a property of mechanisms that guarantees the output of two similar datasets are similar. In particular, we focus on pure differential privacy and not approximate differential privacy.

\begin{definition}[Differential Privacy]
\label{def:dp}
A randomized mechanism $\mathcal{M} : \mathcal{D} \to \mathcal{R}$ is $\epsilon$-differentially private if and only if
$$\Pr[\mathcal{M}(D) \in S] \leq \exp(\epsilon)\Pr[\mathcal{M}(\mathcal{D'}) \in S]$$
for all datasets $D, D' \in \mathcal{D}$ such that $\Vert D - D' \Vert \leq 1$ and all possible sets of outcomes $S \subset \mathcal{R}$.
\end{definition}
To interpret this, we assume that the non-sensitive features and label of an individual are public knowledge, and given distributional knowledge about the dataset, an adversary computes a posterior over the private features of the individual. The private classifier does not provide additional information to the adversary (up to a factor of $\exp(\epsilon)$) beyond what the posterior distribution  provides. 

The nice feature of differential privacy is that it composes across mechanisms. We will use the following basic composition theorem~\cite{robust_statistics,comp}, that directly extends to the setting with public and private features.

\begin{theorem}
\label{thm0}
Let $\mathcal{M}_i: \mathcal{D} \to \mathcal{R}_i$ be an $\epsilon_i$-differentially private algorithm for $i \in \{1, \ldots, n\}$. If $\mathcal{M}_{[n]}: \mathcal{D} \to \prod_{i = 1}^n \mathcal{R}_i$ is defined to be $\mathcal{M}_{[n]}(x) = (\mathcal{M}_1(x), \ldots, \mathcal{M}_n(x))$, then $M_{[n]}$ is $\sum_{i = 1}^n\epsilon_i$-differentially private.
\end{theorem}

A common paradigm for approximating a real-valued function $f: \mathcal{D} \to \mathbb{R}$ with a differentially private mechanism is by adding additive noise calibrated to $f$'s sensitivity, defined as:

\begin{definition}
\label{def:sensitivity}
The $\ell_1$-sensitivity of $f: \mathcal{D} \to \mathbb{R}$ is $\Delta f = \max_{D, D' \in \mathcal{D}, \Vert D - D'\Vert \leq 1}\Vert f(D) - f(D')\Vert_1.$
\end{definition}

Similarly, for a single-dimensional sensitive attribute, we will use the Laplace mechanism that is defined as follows. 

\begin{definition}[The Laplace Mechanism,~\cite{calibrating_noise}]
Given any function $f: \mathcal{D} \to \mathbb{R}$, the Laplace Mechanism is defined as
$$\mathcal{M}_L(x, f(\cdot), \epsilon) = f(x) + Y,$$
where $Y$ is drawn from $\normalfont{\text{Lap}}(\frac{\Delta f}{\epsilon})$.
\end{definition}

\begin{theorem}[\cite{calibrating_noise}]
\label{laplace}
The Laplace mechanism is $\epsilon$-differentially private.
\end{theorem}


\section{Boosting with Random Classifiers ({\sc brc})} 
\label{sec:alg}
We now modify the AdaBoost algorithm (Algorithm~\ref{algo0} in Section~\ref{sec:prelim}) to make it differentially private and whose noise is calibrated to the scenario with both public and private features. We call this algorithm {\em Boosting with Random Classifiers} ({\sc brc}). The pseudocode is shown in Algorithm~\ref{algo1}. 

\begin{algorithm*}[t]
\SetAlgoLined
{\bf Input:} Number of iterations $T$, parameters $\epsilon, c_1, c_2$, observations $\{x_i\}_{i = 1}^n$ and corresponding labels $\{y_i\}_{i = 1}^n$.\\
{\bf Initialization:} Observation weights $w_i^{\text{pub}} = 1, w_i^{\text{pri}} = 1 \quad \forall \{1, \ldots, n\}$.

\For{$t = 1, \ldots, T$}{
Fit a public classifier $h_t^{\text{pub}}(x)$ to the training data using only public features and weights $\{w_i^{\text{pub}}\}_{i = 1}^n$\;
Let $h_t^{\text{pri}}$  be a random classifier with each coefficient and intercept drawn uniformly at random from [-1,1]\;
Compute $\text{err}_t^{\text{pub}} = \frac{\sum_{i = 1}^nw_i^{\text{pub}}\mathbf{1}(y_i \neq h_t^{\text{pub}}(x_i))}{\sum_{i = 1}^n w_i^{\text{pub}}}$ and $\text{err}_t^{\text{pri}} = \frac{\sum_{i = 1}^nw_i^{\text{pri}}\mathbf{1}(y_i \neq h_t^{\text{pri}}(x_i))}{\sum_{i = 1}^n w_i^{\text{pri}}} + \text{Lap}(\frac{c_1c_2T}{\epsilon n})$\;
\eIf{
$|0.5-\normalfont{\text{err}}_t^{\normalfont{\text{pub}}}| > |0.5-\normalfont{\text{err}}_t^{\normalfont{\text{pri}}}|$
}{
Compute $\alpha_t = 0.5-\text{err}_t^{\normalfont{\text{pub}}}$\;
Set $w_i^{\text{pub}} = w_i^{\text{pub}}\exp(\alpha_t \mathbf{1}(y_i \neq h_t^{\text{pub}}(x_i)))\quad \forall i \in \{1, \ldots, n\}$\;
Set $h_t(x) = h_t^{\text{pub}}(x)$\;
}{
Compute $\alpha_t = 0.5-\text{err}_t^{\normalfont{\text{pri}}}$\;
\For{
$i = 1, \ldots, n$
}{
\eIf{
$w_i^{\normalfont{\text{pri}}}\exp(\alpha_t \mathbf{1}(y_i \neq h_t^{\normalfont{\text{pri}}}(x_i))) \leq c_2 \normalfont{\text{ and }} w_i^{\text{pri}}\exp(\alpha_t \mathbf{1}(y_i \neq h_t^{\normalfont{\text{pri}}}(x_i))) \geq \frac{1}{c_1}$
}{
Set $w_i^{\text{pri}} = w_i^{\text{pri}}\exp(\alpha_t \mathbf{1}(y_i \neq h_t^{\text{pri}}(x_i)))$\;
}{
Leave $w_i^{\text{pri}}$ unchanged\;
}
}
Set $h_t(x) = h_t^{\text{pri}}(x)$\;
}
}
{\bf Output:} $H(x) = \text{sign}\left(\sum_{t = 1}^T\alpha_th_t(x)\right)$.
 \caption{Boosting with Random Classifiers ({\sc brc}) Algorithm.}
 \nllabel{algo1}
\end{algorithm*}

In each iteration of \textsc{brc}, we train a non-private classifier with public features, which we call a {\em public classifier}, and generate a random linear classifier for private features, which we call a {\em private classifier}. More formally, the random linear classifier is generated as follows: We assume the range of the private features and the label is known exogenously, and we normalize the data so that it lies in $[-1,1]^d$ without any privacy loss. Each coefficient and the intercept of the random linear classifier is chosen uniformly at random from $[-1, 1]$. 

Next, we perturb the accuracy of the private classifier via noise, and select the better of the public and private classifier based on whose accuracy is farthest from random guessing ($50\%$ accuracy). 
Note that both AdaBoost and \brc assign negative weights to classifiers with accuracy lower than half, which has the effect of flipping the labels, often leading to a strong learner.  

\subsection{Private Weights and Random Classifiers} 
\label{sec:compare}
The key difference with classical AdaBoost as well as previous differentially private versions of AdaBoost~\cite{boosting-dp} is the way in which the weights are used in Algorithm~\ref{algo1}. Note that we maintain two weights for each data point -- one for public features and one for private features. The public classifier (Step 4) is learnt using the public weights in a standard fashion, and these weights are also used to assess the accuracy of this classifier on the data set in Step 6. However, the private classifier $h_t^{\text{pri}}$ in Step 5 is generated {\em without} using the private weights. The private weights are {\em only} used to assess the accuracy of the private classifier in Step 6, and this is the step where we add Laplace noise to make the accuracy computation private.  Since we do not use the private weights for learning in Step 5, we do not need to add noise to the weights themselves, and we argue in Theorem~\ref{main} that this preserves privacy. This novel use of weights is our key innovation and leads our algorithm to use less noise than previous implementations~\cite{boosting-dp}. Indeed, we show empirical gains over these methods in Section~\ref{sec:discuss}.


At a high level, the differentially private algorithm in Step 5 generates $h_t^{\text{pri}}$ from a distribution with sufficient variance that given any (possibly adversarial) private weights, the sampled classifier $h_t^{\text{pri}}$ has accuracy bounded away from $50\%$, leading it to be a weak learner. We attempt to justify this intuition further in Section~\ref{sec:toy} and Section~\ref{sec:allpriv}. 


\noindent {\bf Implementation details.} We set $\alpha_t = \widehat{\text{err}}_t - 0.5$, where $\widehat{\text{err}}_t = \text{err}_t^{\text{pub}}$ in the case where a public classifier is chosen and $\widehat{\text{err}}_t = \text{err}_t^{\text{pri}}$ in the case where a private classifier is chosen. This is in contrast to AdaBoost that uses weights $\log \left(\frac{1 - \widehat{\text{err}}_t}{\widehat{\text{err}}_t}\right)$. We use the former since we need to perturb $\text{err}_t^{\text{pri}}$ with Laplace noise and this might it negative, making the latter choice of weights undefined. In our experiments, we find that the former set of weights suffice to achieve good performance.

We use the hyper-parameters $c_1, c_2$ in a standard fashion to {\em clip} the weights and bound the sensitivity of the classification error to which we add noise. The private weights are always greater than $\frac{1}{c_1}$ and smaller than $c_2$ so that no single data point can have very large influence on the classification error in each iteration, hence bounding the sensitivity of the classification error. In Section~\ref{sec:toy}, we present ``rule of thumb'' parameter settings along with empirical justifications on a toy example.


\subsection{Privacy Analysis}
\label{sec:analysis}

\begin{theorem}
\label{main}
Algorithm~\ref{algo1} is $\epsilon$-differentially private.
\end{theorem}

To prove Theorem~\ref{main}, we first provide a bound on the sensitivity of accuracy. Formally, let $\mathcal{H}^{\text{rand}}$ be the set of possible random linear classifiers, and $g : \mathcal{H}^{\text{rand}} \times \mathcal{D} \to \mathbb{R}$ be a function which maps from a private classifier (a random linear classifier) $h^{\text{pri}} \in \mathcal{H}^{\text{rand}}$ and a dataset $D \in \mathcal{D}$ to the accuracy of $h^{\text{pri}}$ on $D$; {\em i.e.} $g(h^{\text{pri}}, D)$ is the accuracy of classifier $h^{\text{pri}}$ on dataset $D$. 

We present the following key lemma, and Theorem~\ref{main} follows almost immediately from it.

\begin{lemma}
\label{sensitivity}
$\Delta g \leq \frac{c_1c_2}{n}$, where $n$ is the size of the dataset.
\end{lemma}

\begin{proof}
Let $D, D'$ be two datasets such that $\Vert D - D' \Vert \leq 1$. Suppose $D$ consists of observations $\{x_i\}_{i = 1}^n$ and labels $\{y_i\}_{i = 1}^n$, and $D'$ consists of $\{x_i\}_{i = 1}^{n - 1} \cup \{x_n'\}$ and $\{y_i\}_{i = 1}^{n - 1} \cup \{y_n'\}$, where the two datasets differ only on the private features of the $n^{\text{th}}$ data point. Suppose $h^{\text{pri}}$ is some private classifier, $w_i^{\text{pri}}$ is the current observation weight for the data point $(x_i, y_i)$, and $w_n'^{\text{ pri}}$ is the current observation weight for the data point $(x_n', y_n')$. Then,
\begin{align*}
|g(h^{\text{pri}}, D) - g(h^{\text{pri}}, D')| \leq &\bigg|\frac{\sum_{i = 1}^nw_i^{\text{pri}}\mathbf{1}(y_i \neq h_t^{\text{pri}}(x_i))}{\sum_{i = 1}^n w_i^{\text{pri}}}\\ 
&- \frac{\sum_{i = 1}^{n - 1}w_i^{\text{pri}}\mathbf{1}(y_i \neq h_t^{\text{pri}}(x_i)) + w_n'^{\text{ pri}} \mathbf{1}(y_n' \neq h_t^{\text{pri}}(x_n'))}{\sum_{i = 1}^{n - 1} w_i^{\text{pri}} + w_n'^{\text{ pri}}}\bigg|.
\end{align*}

To bound this quantity, we consider the following cases.

\paragraph{Case 1}If $y_n = h_t^{\text{pri}}(x_n)$ and $y_n' = h_t^{\text{pri}}(x_n')$, we have
\begin{align*}
&\bigg|\frac{\sum_{i = 1}^nw_i^{\text{pri}}\mathbf{1}(y_i \neq h_t^{\text{pri}}(x_i))}{\sum_{i = 1}^n w_i^{\text{pri}}}  - \frac{\sum_{i = 1}^{n - 1}w_i^{\text{pri}}\mathbf{1}(y_i \neq h_t^{\text{pri}}(x_i)) + w_n'^{\text{ pri}} \mathbf{1}(y_n' \neq h_t^{\text{pri}}(x_n'))}{\sum_{i = 1}^{n - 1} w_i^{\text{pri}} + w_n'^{\text{ pri}}}\bigg|\\
=&\left|\frac{w_n'^{\text{ pri}} \sum_{i = 1}^{n - 1}w_i^{\text{pri}}\mathbf{1}(y_i \neq h_t^{\text{pri}}(x_i)) - w_n^{\text{pri}}\sum_{i = 1}^{n - 1}w_i^{\text{pri}}\mathbf{1}(y_i \neq h_t^{\text{pri}}(x_i))}{(\sum_{i = 1}^nw_i^{\text{pri}})(\sum_{i = 1}^{n - 1} w_i^{\text{pri}} + w_n'^{\text{ pri}})}\right| \\
\leq & \left|\frac{w_n'^{\text{ pri}} - w_n^{\text{pri}}}{\sum_{i = 1}^n w_i^{\text{pri}}}\right| \leq \frac{\max(w_n^{\text{ pri}}, w_n'^{\text{ pri}})}{\sum_{i = 1}^n w_i^{\text{pri}}}\leq \frac{c_1c_2}{n},
\end{align*}
where the second step is because $\sum_{i = 1}^{n - 1}w_i^{\text{pri}}\mathbf{1}(y_i \neq h_t^{\text{pri}}(x_i)) \leq \sum_{i = 1}^{n - 1} w_i^{\text{pri}} + w_n'^{\text{ pri}}$, and the last step is because all the private weights are no greater than $c_2$ and no smaller than $\frac{1}{c_1}$.

\paragraph{Case 2} If $y_n \neq h_t^{\text{pri}}(x_n)$ and $y_n' \neq h_t^{\text{pri}}(x_n')$, we similarly have
\begin{align*}
&\bigg|\frac{\sum_{i = 1}^nw_i^{\text{pri}}\mathbf{1}(y_i \neq h_t^{\text{pri}}(x_i))}{\sum_{i = 1}^n w_i^{\text{pri}}} - \frac{\sum_{i = 1}^{n - 1}w_i^{\text{pri}}\mathbf{1}(y_i \neq h_t^{\text{pri}}(x_i)) + w_n'^{\text{ pri}} \mathbf{1}(y_n' \neq h_t^{\text{pri}}(x_n'))}{\sum_{i = 1}^{n - 1} w_i^{\text{pri}} + w_n'^{\text{ pri}}}\bigg|\\
=&\left|\frac{w_n^{\text{pri}} \sum_{i = 1}^{n - 1}w_i^{\text{pri}}\mathbf{1}(y_i = h_t^{\text{pri}}(x_i)) - w_n'^{\text{ pri}}\sum_{i = 1}^{n - 1}w_i^{\text{pri}}\mathbf{1}(y_i = h_t^{\text{pri}}(x_i))}{(\sum_{i = 1}^nw_i^{\text{pri}})(\sum_{i = 1}^{n - 1} w_i^{\text{pri}} + w_n'^{\text{ pri}})}\right|\\
\leq &\left|\frac{w_n^{\text{ pri}} - w_n'^{\text{ pri}}}{\sum_{i = 1}^n w_i^{\text{pri}}}\right| \leq \frac{\max(w_n^{\text{pri}}, w_n'^{\text{ pri}})}{\sum_{i = 1}^n w_i^{\text{pri}}}\leq \frac{c_1c_2}{n}.
\end{align*}

\paragraph{Case 3} If $y_n = h_t^{\text{pri}}(x_n)$ and $y_n' \neq h_t^{\text{pri}}(x_n')$, we similarly have

\begin{align*}
&\bigg|\frac{\sum_{i = 1}^nw_i^{\text{pri}}\mathbf{1}(y_i \neq h_t^{\text{pri}}(x_i))}{\sum_{i = 1}^n w_i^{\text{pri}}} - \frac{\sum_{i = 1}^{n - 1}w_i^{\text{pri}}\mathbf{1}(y_i \neq h_t^{\text{pri}}(x_i)) + w_n'^{\text{ pri}} \mathbf{1}(y_n' \neq h_t^{\text{pri}}(x_n'))}{\sum_{i = 1}^{n - 1} w_i^{\text{pri}} + w_n'^{\text{ pri}}}\bigg|\\
=& \frac{w_n^{\text{pri}}\sum_{i = 1}^{n - 1}w_i^{\text{pri}}\mathbf{1}(y_i \neq h_t^{\text{pri}}(x_i)) + w_n'^{\text{ pri}}\sum_{i = 1}^{n - 1}w_i^{\text{pri}}\mathbf{1}(y_i = h_t^{\text{pri}}(x_i)) + w_n^{\text{pri}}w_n'^{\text{ pri}}}{(\sum_{i = 1}^nw_i^{\text{pri}})(\sum_{i = 1}^{n - 1} w_i^{\text{pri}} + w_n'^{\text{ pri}})}\\
\leq &\frac{\max(w_n^{\text{pri}}, w_n'^{\text{ pri}})}{\sum_{i = 1}^n w_i^{\text{pri}}} \leq \frac{c_1c_2}{n}.
\end{align*}

\paragraph{Case 4} If $y_n = h_t^{\text{pri}}(x_n)$ and $y_n' \neq h_t^{\text{pri}}(x_n')$, we similarly have
\begin{align*}
&\bigg|\frac{\sum_{i = 1}^nw_i^{\text{pri}}\mathbf{1}(y_i \neq h_t^{\text{pri}}(x_i))}{\sum_{i = 1}^n w_i^{\text{pri}}} - \frac{\sum_{i = 1}^{n - 1}w_i^{\text{pri}}\mathbf{1}(y_i \neq h_t^{\text{pri}}(x_i)) + w_n'^{\text{ pri}} \mathbf{1}(y_n' \neq h_t^{\text{pri}}(x_n'))}{\sum_{i = 1}^{n - 1} w_i^{\text{pri}} + w_n'^{\text{ pri}}}\bigg|\\
=& \frac{w_n^{\text{pri}}\sum_{i = 1}^{n - 1}w_i^{\text{pri}}\mathbf{1}(y_i = h_t^{\text{pri}}(x_i)) + w_n'^{\text{ pri}}\sum_{i = 1}^{n - 1}w_i^{\text{pri}}\mathbf{1}(y_i \neq h_t^{\text{pri}}(x_i)) + w_n^{\text{pri}}w_n'^{\text{ pri}}}{(\sum_{i = 1}^nw_i^{\text{pri}})(\sum_{i = 1}^{n - 1} w_i^{\text{pri}} + w_n'^{\text{ pri}})}\\
\leq &\frac{\max(w_n^{\text{pri}}, w_n'^{\text{ pri}})}{\sum_{i = 1}^n w_i^{\text{pri}}} \leq \frac{c_1c_2}{n}.
\end{align*}

Hence, we always have 
$$|g(h^{\text{pri}}, D) - g(h^{\text{pri}}, D')| \leq \frac{c_1c_2}{n},$$
from which we can conclude that $\Delta g \leq \frac{c_1c_2}{n}$.

\end{proof}

\begin{proof}[Proof of Theorem~\ref{main}]
By Theorem~\ref{laplace}, computing the perturbed accuracy of a private classifier is $(\Delta g \cdot \frac{\epsilon n}{c_1c_2T})$-differentially private. By Lemma~\ref{sensitivity}, we have $\Delta g \leq \frac{c_1c_2}{n}$, so this step is $(\frac{\epsilon}{T})$-differentially private. Hence, by Theorem~\ref{thm0}, computing the perturbed accuracy of private classifiers for $T$ iterations is $\epsilon$-differentially private. Hence, our algorithm is $\epsilon$-differentially private.
\end{proof}

\section{A Toy Example and Parameter Settings} 
\label{sec:toy}
To generate more intuition for the algorithm in Section~\ref{sec:alg},  we present results on a toy example. This not only informs how to set the parameters for {\sc brc} in our experiments (Section~\ref{sec:experiments}), but also shows that there are two aspects that make \brc perform empirically well. 

\begin{enumerate}
    \item In the absence of privacy considerations, that is, with $\epsilon \rightarrow \infty$,  both Algorithm~\ref{algo1} and AdaBoost (Algorithm~\ref{algo0}) have good performance even if the base classifiers are randomly chosen ignoring weights of data points (Step 5 in Algorithm~\ref{algo1}), and even if the weights are clipped (Step 14).
    \item If $T, c_1, c_2$ are chosen such that $\frac{c_1 c_2 T}{\epsilon n} \ll 1$, then the impact of adding Laplace noise in Step 6 is minimal. We can choose $T, c_1, c_2$ to ensure this for reasonable privacy budgets $\epsilon$.
\end{enumerate}

\noindent {\bf Toy Example.} We demonstrate these points on a toy example. The input is one-dimensional with $n$ points; assume this is the private feature. There is an irrelevant public feature whose attribute values are uncorrelated with the label, and we ignore this feature. The left $n/2$ points are labeled ``-'' and the right $n/2$ points are labeled ``+''. On such an instance, Algorithm~\ref{algo0} using weights to find the best classifier each step will trivially converge to the overall optimal classifier, which is the hyperplane separating the left points from the right points.

Now we run \brc with the following settings: $n = 2000$, $c_1 = c_2 = 2$, and $T = 50$. The classifiers in Step 5 are chosen using the following procedure: At each iteration, randomly flip each label with probability $p \approx 0.49$ and find the optimal classifier treating each point as having unit weight. This simple procedure effectively destroys label information, and the resulting classifier at each step is close to random. We plot the histogram of these classifiers in Figure~\ref{fig7}; this distribution is close to uniform, showing that the process essentially picks classifiers uniformly at random. 

\begin{figure*}[h]
 \centering
    \begin{subfigure}[b]{0.32\linewidth}
    \centering  \includegraphics[width=0.9\linewidth]{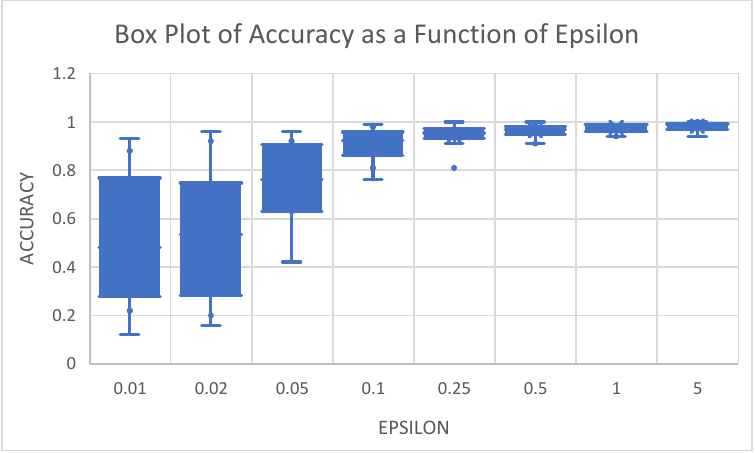}
    \subcaption{\label{fig6a}Plot of accuracy of \brc v.s. $\epsilon$.}
    \end{subfigure}
    \hfill
    \begin{subfigure}[b]{0.32\linewidth}
    \centering  \includegraphics[width=\linewidth]{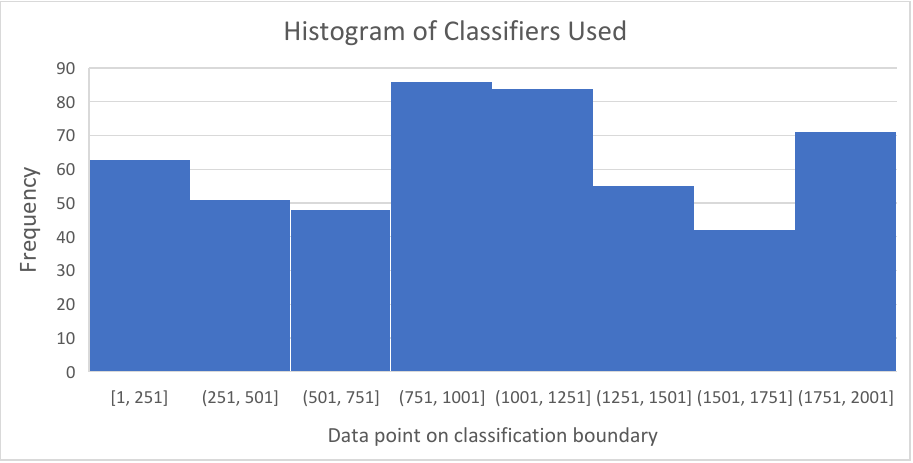}
    \subcaption{\label{fig7} Histogram of  classifiers.}
    \end{subfigure}
    \hfill
    \begin{subfigure}[b]{0.32\linewidth}
    \centering  \includegraphics[width=\linewidth]{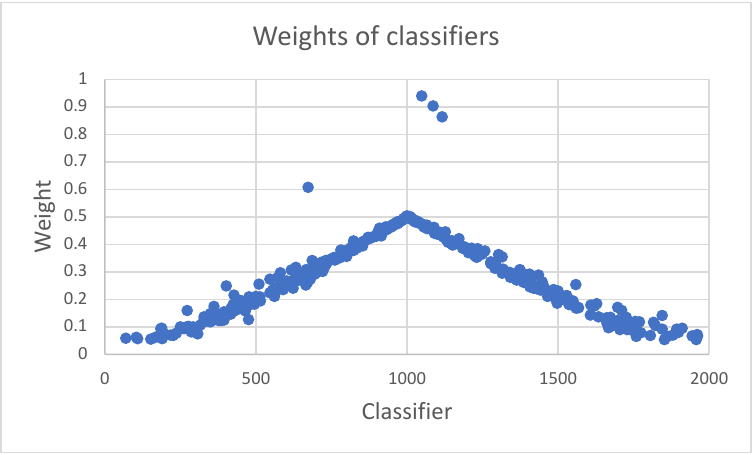}
    \subcaption{\label{fig8} Weights $\alpha$ for each classifier.}
    \end{subfigure}
    \caption{The toy example has $1,000$ points labeled ``-'' on a line followed by $1,000$ points labeled ``+''. In these plot, a classifier at data point $i$ means it classifies all points $x < i$ as ``-'' and all points $x \ge i$ as ``+''. The histogram in the center plots the number of classifiers for each bucket of $i$ across iterations of Algorithm~\ref{algo1}. The RHS plot shows the weight $\alpha$ assigned to each classifier, assuming $\epsilon = 100$, which makes the noise in Step 6 irrelevant.}
    
\end{figure*}

On this instance, we run \brc $10$ times for each choice of $\epsilon$. Note that $\frac{c_1 c_2 T}{n} = 0.1$, so that we expect that for roughly $\epsilon \ge 0.1$, \brc should work well, and below this threshold, it should behave erratically. We show a Box plot of the accuracy as a function of $\epsilon$ in Figure~\ref{fig6a}. validating the above statement. 

Note that when $\epsilon$ is large, then the noise added in Step 6 becomes irrelevant, and we recover a version of AdaBoost with weight clipping and using random classifiers. We show the weights that get assigned to the classifiers in Figure~\ref{fig8} when $\epsilon = 100$. Despite the classifiers being randomly generated in Step 5 (and hence, individually having low accuracy), their final weights $\alpha$ concentrate around the mid-point of $1,000$, leading to the final classification having very high accuracy.

\noindent {\bf Parameter Settings in Algorithm~\ref{algo1}.} 
The above toy example and intuition provides a simple rule of thumb for setting parameters: We should set $T$ depending on the data size $n$ and privacy budget $\epsilon$ so that $\frac{c_1c_2T}{\epsilon n}$ is at most a small constant, say $1/2$. This will inform the parameters for our experiments in Section~\ref{sec:experiments}.

\section{Empirical Results}
\label{sec:experiments}
In this section, we present some empirical results for the \brc algorithm. We use four binary classification datasets: Adult Income dataset~\cite{adult}, Default of Credit Car Clients dataset~\cite{credit}, Sepsis Survival Minimal Clinical Records dataset~\cite{sepsis}, and Bank Marketing dataset~\cite{bank}. For each dataset, we set some features as public and others as private. This choice is based on ``best guesses'' and is necessarily somewhat arbitrary.  Note that different countries and different agencies might have different benchmarks for privacy, and our choice of private features is by no means canonical. We believe that \brc should perform similarly well for other choices of public and private features.

\subsection{Parameter Settings}
For each dataset, we randomly select 10\% of it as the test set, and use the remaining as the training set. We set the number of iterations as $T = 25$ and the parameters $c_1 = c_2 = \sqrt{2}$ for all four datasets. $T$ is chosen as such because \brc usually converges in around 20 to 25 iterations on the datasets we use; the choice of $\sqrt{2}$ for $c_1$ and $c_2$ turns out to provide a good balance between privacy and accuracy. We experiment with five values of privacy budget, namely $\epsilon = 0.01, 0.02, 0.04, 0.08$ and $0.16$. 
Since the dataset sizes are larger than in our toy example above, while $T$, $c_1$, and $c_2$ are smaller, the values of $\epsilon$ for which we achieve good accuracy becomes smaller.
By Theorem~\ref{main}, \brc ensures $\epsilon$-differential privacy. 

\subsection{Baseline Algorithms and Preprocessing}
We compare our results against three baseline classifiers: The first is the {\em non-private baseline}, {\em i.e.} the accuracy of non-private logistic regression trained on all features; the second is the {\em private baseline}, {\em i.e.} the accuracy of differentially private logistic regression trained on all features; the third is the {\em public feature baseline}, {\em i.e.} the accuracy of non-private logistic regression trained on public features only. The first baseline yields an upper bound on accuracy achievable, while the latter two yield lower bounds. Our goal is to demonstrate that \brc supersedes the second and the third baseline. Indeed, we show that \brc achieves higher accuracy than the second baseline for all values of $\epsilon$ and achieves higher accuracy than the third baseline most values of $\epsilon$. Each accuracy is obtained by taking average across ten experiments.

\begin{figure*}[h]
    \centering
    \begin{subfigure}[b]{0.24\linewidth}
    \centering  \begin{tikzpicture}[yscale = 0.34, xscale = 0.34][!h]
\begin{axis}[
    legend style={at={(0.6,0.33)}},
    ymin = 0, ymax = 1,
    very thick,
    axis lines = left,
    xmode=log,
    xlabel = \(\epsilon\) {(\(T = 25, c_1 = c_2 = \sqrt{2}\))},
    ylabel = {\(\text{Accuracy}\)},    
    xtick = {0.01, 0.02, 0.04, 0.08, 0.16},
    xticklabels = {0.01, 0.02, 0.04, 0.08, 0.16},
]
\addplot [
    very thick,
    domain = 0.01:0.16,
    color = blue,
    mark = *,
    error bars/.cd, 
    y fixed,
    y dir=both, 
    y explicit    
]table[x=x, y=y,y error=error, col sep=comma] {
x, y, error
0.01, 0.5833, 0.10276637643421352
0.02, 0.6617, 0.03696339855701698
0.04, 0.7034, 0.019225346377164696
0.08, 0.7187, 0.008686005236905735
0.16, 0.7275, 0.004534542263973636
};
\addlegendentry{\brc}

\addplot [
    very thick,
    dotted,
    domain=0.01:0.16, 
    samples=100, 
    color=green,
]
{0.7575};
\addlegendentry{Non-Private Baseline}

\addplot [
    very thick,
    dotted,
    domain = 0.01:0.16,
    color =red,
    mark = square,
    error bars/.cd, 
    y fixed,
    y dir=both, 
    y explicit    
]table[x=x, y=y,y error=error, col sep=comma] {
x, y, error
0.01, 0.5285, 0.07512066772033914
0.02, 0.530, 0.07034649586191694
0.04, 0.5875, 0.06526067079155243
0.08, 0.6309, 0.035914230935871774
0.16, 0.6815, 0.034871676607393445
};
\addlegendentry{Private Baseline}

\addplot [
    very thick,
    dotted,
    domain=0.01:0.16, 
    samples=100, 
    color=cyan,
]
{0.6159};
\addlegendentry{Public Feature Baseline}

\end{axis}
\end{tikzpicture}
    \subcaption{\label{fig1}Adult Income Dataset}
    \end{subfigure}
    \hfill
    \begin{subfigure}[b]{0.24\linewidth}
    \centering  \begin{tikzpicture}[yscale = 0.34, xscale = 0.34][!h]
\begin{axis}[
    legend style={at={(0.6,0.33)}},
    ymin = 0, ymax = 1,
    very thick,
    axis lines = left,
    xmode=log,
    xlabel = \(\epsilon\) {(\(T = 25, c_1 = c_2 = \sqrt{2}\))},
    ylabel = {\(\text{Accuracy}\)},    
    xtick = {0.01, 0.02, 0.04, 0.08, 0.16},
    xticklabels = {0.01, 0.02, 0.04, 0.08, 0.16},
]
\addplot [
    very thick,
    domain = 0.01:0.16,
    color = blue,
    mark = *,
    error bars/.cd, 
    y fixed,
    y dir=both, 
    y explicit    
]table[x=x, y=y,y error=error, col sep=comma] {
x, y, error
0.01, 0.5241, 0.06801412244587907
0.02, 0.6021, 0.05257739348246776
0.04, 0.6727, 0.04047283900053165
0.08, 0.7115, 0.03134954749785891
0.16, 0.7338, 0.01272322991803198
};
\addlegendentry{\brc}

\addplot [
    very thick,
    dotted,
    domain=0.01:0.16, 
    samples=100, 
    color=green,
]
{0.8298};
\addlegendentry{Non-Private Baseline}

\addplot [
    very thick,
    dotted,
    domain = 0.01:0.16,
    color =red,
    mark = square,
    error bars/.cd, 
    y fixed,
    y dir=both, 
    y explicit    
]table[x=x, y=y,y error=error, col sep=comma] {
x, y, error
0.01, 0.5060, 0.08534888979110947
0.02, 0.5236, 0.0712347173281054
0.04, 0.5651, 0.06813335663141522
0.08, 0.6252, 0.0578112860677855
0.16, 0.6757, 0.044356074618669486
};
\addlegendentry{Private Baseline}

\addplot [
    very thick,
    dotted,
    domain=0.01:0.16, 
    samples=100, 
    color=cyan,
]
{0.5678};
\addlegendentry{Public Feature Baseline}

\end{axis}
\end{tikzpicture}
    \subcaption{\label{fig2}Credit Cards Dataset}
    \end{subfigure}
    \hfill
    \begin{subfigure}[b]{0.24\linewidth}
    \centering  \begin{tikzpicture}[yscale = 0.34, xscale = 0.34][!h]
\begin{axis}[
    legend style={at={(0.6,0.33)}},
    ymin = 0, ymax = 1,
    very thick,
    axis lines = left,
    xmode=log,
    xlabel = \(\epsilon\) {(\(T = 25, c_1 = c_2 = \sqrt{2}\))},
    ylabel = {\(\text{Accuracy}\)},    
    xtick = {0.01, 0.02, 0.04, 0.08, 0.16},
    xticklabels = {0.01, 0.02, 0.04, 0.08, 0.16},
]

\addplot [
    very thick,
    domain = 0.01:0.16,
    color = blue,
    mark = *,
    error bars/.cd, 
    y fixed,
    y dir=both, 
    y explicit    
]table[x=x, y=y,y error=error, col sep=comma] {
x, y, error
0.01, 0.5830, 0.07833768837800979
0.02, 0.6666, 0.03722252697793115
0.04, 0.6840, 0.010783636395672716
0.08, 0.6911, 0.007320381830219332
0.16, 0.6964, 0.0038668560010400527
};
\addlegendentry{\brc}

\addplot [
    very thick,
    dotted,
    domain=0.01:0.16, 
    samples=100, 
    color=green,
]
{0.6983};
\addlegendentry{Non-Private Baseline}

\addplot [
    very thick,
    dotted,
    domain = 0.01:0.16,
    color =red,
    mark = square,
    error bars/.cd, 
    y fixed,
    y dir=both, 
    y explicit    
]table[x=x, y=y,y error=error, col sep=comma] {
x, y, error
0.01, 0.5329, 0.09134767648124437
0.02, 0.5694, 0.08655640919512429
0.04, 0.6190, 0.08288973949316225
0.08, 0.6453, 0.04710192355013991
0.16, 0.6829, 0.009331891362941096
};
\addlegendentry{Private Baseline}

\addplot [
    very thick,
    dotted,
    domain=0.01:0.16, 
    samples=100, 
    color=cyan,
]
{0.5416};
\addlegendentry{Public Feature Baseline}

\end{axis}
\end{tikzpicture}
    \subcaption{\label{fig3}Sepsis Survival Dataset}
    \end{subfigure}
    \hfill
\begin{subfigure}[b]{0.25\linewidth}
    \centering   \begin{tikzpicture}[yscale = 0.34, xscale = 0.34][!h]
\begin{axis}[
    legend style={at={(0.6,0.33)}},
    ymin = 0, ymax = 1,
    very thick,
    axis lines = left,
    xmode=log,
    xlabel = \(\epsilon\) {(\(T = 25, c_1 = c_2 = \sqrt{2}\))},
    ylabel = {\(\text{Accuracy}\)},    
    xtick = {0.01, 0.02, 0.04, 0.08, 0.16},
    xticklabels = {0.01, 0.02, 0.04, 0.08, 0.16},
]
\addplot [
    very thick,
    domain = 0.01:0.16,
    color = blue,
    mark = *,
    error bars/.cd, 
    y fixed,
    y dir=both, 
    y explicit    
]table[x=x, y=y,y error=error, col sep=comma] {
x, y, error
0.01, 0.6026, 0.10154787933911355
0.02, 0.7356, 0.06085767660901526
0.04, 0.7860, 0.04005197790677513
0.08, 0.8224, 0.021827100712986536
0.16, 0.8545, 0.019092861553153095
};
\addlegendentry{\brc}

\addplot [
    very thick,
    dotted,
    domain=0.01:0.16, 
    samples=100, 
    color=green,
]
{0.9688};
\addlegendentry{Non-Private Baseline}

\addplot [
    very thick,
    dotted,
    domain = 0.01:0.16,
    color =red,
    mark = square,
    error bars/.cd, 
    y fixed,
    y dir=both, 
    y explicit    
]table[x=x, y=y,y error=error, col sep=comma] {
x, y, error
0.01, 0.5075, 0.13482951578668553
0.02, 0.5810, 0.0991639153945424
0.04, 0.6408, 0.07388057513861102
0.08, 0.6946, 0.054667693309617196
0.16, 0.7687, 0.03493426519455197
};
\addlegendentry{Private Baseline}

\addplot [
    very thick,
    dotted,
    domain=0.01:0.16, 
    samples=100, 
    color=cyan,
]
{0.6711};
\addlegendentry{Public Feature Baseline}

\end{axis}
\end{tikzpicture}
    \subcaption{\label{fig4}Bank Marketing Dataset}
    \end{subfigure}    
    \caption{Accuracy of the \brc algorithm compared to baselines on our datasets.}
\end{figure*}
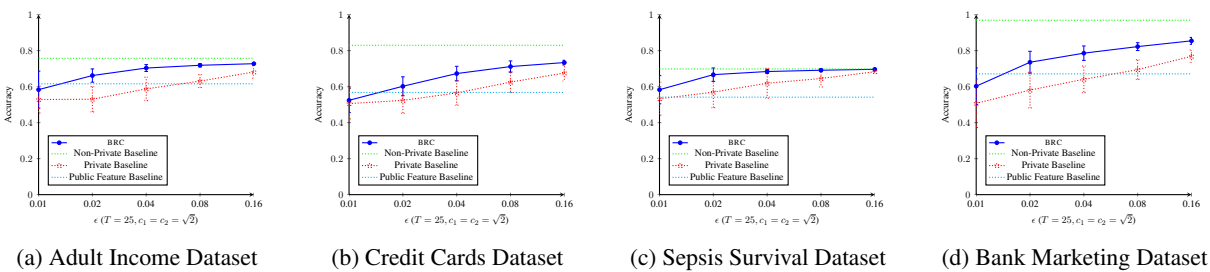

Finally, all these datasets are unbalanced, meaning that there are many more data points of one label compared to the opposite label. This leads to naive classifiers (e.g. classifiers that simply predict all positives or all negatives) giving high accuracy. We therefore balance each dataset by taking a random subset of each dataset with equal number of positive labels and negative labels. Note that this step is only to make the datasets more difficult to classify and hence show the empirical improvement in accuracy achievable by the \brc algorithm. Our algorithm itself and the privacy guarantees are oblivious to any balance properties of the data. 

\subsection{Results}
We present the results on the datasets mentioned above.

\noindent {\bf Adult Income Dataset} The Adult Income dataset has 15 attributes around 50000 individuals. We want to predict whether the income of each individual is higher than \$50,000 or not. 
We set ``workclass", ``fnlwgt", ``race", ``sex", and ``native country" as pubic features. The others are private. The results are presented in Figure~\ref{fig1}.

In this experiment, the non-private baseline is 75.75\% and the public feature baseline is 61.59\%. The \brc algorithm consistently achieves higher accuracy than the private baseline, and exceeds the public feature baseline for $\epsilon \geq 0.02$. Moreover, it reaches roughly 73\% accuracy for $\epsilon = 0.16$, which almost recovers the non-private baseline.

\noindent {\bf Default of Credit Card Clients Dataset} This dataset records information on default payments, demographic factors, credit data, history of payment, and bill statements of credit card clients in Taiwan from April 2005 to September 2005 for roughly 30,000 customers and aims to predict whether a customer has default payment next month.  
We set ``amount of the given credit" and ``gender" as public and the others as private. The results are presented in Figure~\ref{fig2}.

In this experiment, the non-private baseline is 82.98\% and the public feature baseline is 56.78\%. The \brc algorithm still performs better than the private baseline for all values of $\epsilon$ and supersedes the public feature baseline for $\epsilon \geq 0.02$.

\noindent {\bf Sepsis Survival Minimal Clinical Records Dataset} This contains health records of more than 110000 patients who had sepsis, and aims to predict whether a patient with sepsis can survive. We set ``sex" and ``episode number" as public and ``age" as private. The results are  in Figure~\ref{fig3}.

The non-private baseline is 69.83\% and the public feature baseline is 54.16\%. The \brc algorithm has higher accuracy than both the private baseline and the public feature baseline for all values of $\epsilon$. Moreover, \brc achieves 69\% accuracy for $\epsilon = 0.08$, reaching the non-private baseline.

\noindent {\bf Bank Marketing Dataset} This dataset contains information about marketing campaigns of a Portuguese banking institution. It tabulates 16 attributes for more than 40000 bank clients, and the classification goal is to predict if the client will subscribe a term deposit. 
We set ``number of contacts performed", ``number of days to the previous contact", and ``outcome of the previous campaign" as public. The results are presented in Figure~\ref{fig4}.

In this experiment, the non-private baseline is 96.88\% and the public feature baseline is 67.11\%. The \brc algorithm performs better than the private baseline for all values of $\epsilon$, and supersedes the public feature baseline for $\epsilon \geq 0.02$. 

\section{Boosting for Private Linear Classification}
\label{sec:allpriv}
The framework in Algorithm~\ref{algo1} provides an alternate algorithm for differentially private linear classification when all features and the label are private, that is, the standard setting for differential privacy. 
The procedure is presented in Algorithm~\ref{algo2}, which is a straightforward adaptation of Algorithm~\ref{algo1} to the setting where there's no public features. Recall that the only reason we assume the label is public is to train the public classifier in step (4) without noise, so we don't need the public-label assumption.



\begin{algorithm*}[h]
\SetAlgoLined
{\bf Input:} Number of iterations $T$, privacy parameters $ \epsilon, c_1, c_2$, observations $\{x_i\}_{i = 1}^n$ and corresponding labels $\{y_i\}_{i = 1}^n$, all assumed to be private.\\
{\bf Initialization:} Set $ w_i = 1 \quad \forall \{1, \ldots, n\}$.

\For{$t = 1, \ldots, T$}{
Let $h_t$  be a random classifier with each coefficient and intercept drawn uniformly at random from [-1,1]\;
Compute  $\text{err}_t = \frac{\sum_{i = 1}^n w_i \mathbf{1}(y_i \neq h_t (x_i))}{\sum_{i = 1}^n w_i } + \text{Lap}(\frac{c_1c_2T}{\epsilon n})$\;
Compute $\alpha_t = 0.5-\text{err}_t$\;
\For{
$i = 1, \ldots, n$
}{
\eIf{
$w_i \exp(\alpha_t \mathbf{1}(y_i \neq h_t (x_i))) \leq c_2$ \normalfont{ and }  $w_i \exp(\alpha_t \mathbf{1}(y_i \neq h_t (x_i))) \geq \frac{1}{c_1}$
}{
Set $w_i = w_i \exp(\alpha_t \mathbf{1}(y_i \neq h_t (x_i)))$\;
}{
Leave $w_i$ unchanged\;
}
}
}
{\bf Output:} $H(x) = \text{sign}\left(\sum_{t = 1}^T\alpha_t h_t(x)\right)$.
 \caption{\brc for Private Linear Classification.}
 \nllabel{algo2}
\end{algorithm*}

In Algorithm~\ref{algo2}, we could alternatively assume the random classifiers have intercept zero, since this is the expected value of $h(x_i)$ when each coefficient is chosen i.i.d. from $[-1,1]$. Both these approaches lead to roughly the same accuracy in our experiments. A more principled way to choose the intercept that applies the exponential mechanism is as follows: We first observe that there are at most $n$ different values of intercept w.l.o.g. For these intercepts, we take the accuracy as the utility, and apply the Exponential mechanism to choose an intercept. This will ensure differential privacy. 

The privacy guarantee of Algorithm~\ref{algo2} is presented in Theorem~\ref{main2}, which follows directly from Theorem~\ref{main} since Algorithm~\ref{algo2} is a direct extension of Algorithm~\ref{algo1}.

\begin{theorem}
\label{main2}
Algorithm~\ref{algo2} is $\epsilon$-differentially private.
\end{theorem}

\noindent {\bf Empirical Results.}
We next compare Algorithm~\ref{algo2} with the private baseline, {\em i.e.} running the differentially private logistic regression algorithm in~\cite{logistic}, on our datasets.  We assume that all the features and labels are private. We perform the same pre-processing as in Section~\ref{sec:experiments} to balance the labels, and use the same parameter settings for $c_1, c_2, T$, so that $c_1 = c_2 = \sqrt{2}$, and $T = 25$. The results for the four datasets are presented in Figure~\ref{fig_alterante}. We see that \brc performs consistently better than the private baseline~\cite{logistic} on all datasets. Therefore, \brc provides an alternate way to find a differentially private linear classifier that has comparable or even better accuracy than differentially private logistic regression~\cite{logistic}, with a much simpler privacy analysis and implementation. 

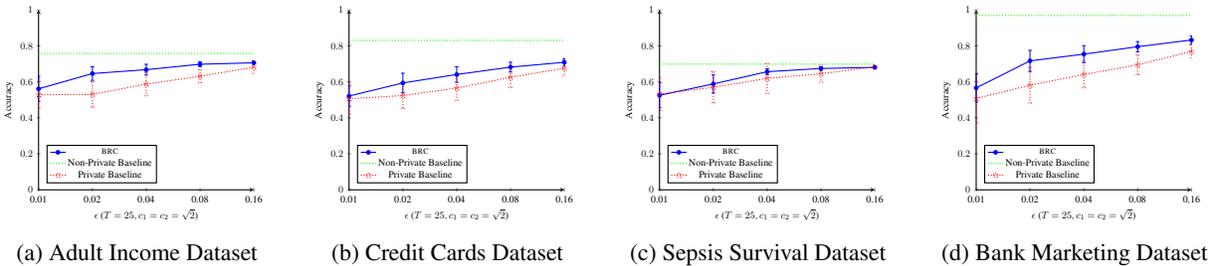
\begin{figure*}[htbp]
\centering
\begin{subfigure}[t]{0.24\linewidth}
    \centering   \begin{tikzpicture}[yscale = 0.34, xscale = 0.34][!h]
\begin{axis}[
    legend style={at={(0.55,0.25)}},
    ymin = 0, ymax = 1,
    very thick,
    axis lines = left,
    xmode=log,
    xlabel = \(\epsilon\) {(\(T = 25, c_1 = c_2 = \sqrt{2}\))},
    ylabel = {\(\text{Accuracy}\)},    
    xtick = {0.01, 0.02, 0.04, 0.08, 0.16},
    xticklabels = {0.01, 0.02, 0.04, 0.08, 0.16},
]

\addplot [
    very thick,
    domain = 0.01:0.16,
    color = blue,
    mark = *,
    error bars/.cd, 
    y fixed,
    y dir=both, 
    y explicit    
]table[x=x, y=y,y error=error, col sep=comma] {
x, y, error
0.01, 0.5611, 0.06956415098141912
0.02, 0.6459, 0.03771960419841049
0.04, 0.6675, 0.02994855068879492
0.08, 0.6977, 0.013076432951485934
0.16, 0.7057, 0.009908251316044671
};
\addlegendentry{\brc}

\addplot [
    very thick,
    dotted,
    domain=0.01:0.16, 
    samples=100, 
    color=green,
]
{0.7575};
\addlegendentry{Non-Private Baseline}

\addplot [
    very thick,
    dotted,
    domain = 0.01:0.16,
    color =red,
    mark = square,
    error bars/.cd, 
    y fixed,
    y dir=both, 
    y explicit    
]table[x=x, y=y,y error=error, col sep=comma] {
x, y, error
0.01, 0.5285, 0.07512066772033914
0.02, 0.530, 0.07034649586191694
0.04, 0.5875, 0.06526067079155243
0.08, 0.6309, 0.035914230935871774
0.16, 0.6815, 0.034871676607393445
};
\addlegendentry{Private Baseline}

\end{axis}
\end{tikzpicture}
    \subcaption{\label{fig9}Adult Income Dataset}
    \end{subfigure} 
        \hfill
\begin{subfigure}[t]{0.24\linewidth}
    \centering   \begin{tikzpicture}[yscale = 0.34, xscale = 0.34][!h]
\begin{axis}[
    legend style={at={(0.55,0.25)}},
    ymin = 0, ymax = 1,
    very thick,
    axis lines = left,
    xmode=log,
    xlabel = \(\epsilon\) {(\(T = 25, c_1 = c_2 = \sqrt{2}\))},
    ylabel = {\(\text{Accuracy}\)},    
    xtick = {0.01, 0.02, 0.04, 0.08, 0.16},
    xticklabels = {0.01, 0.02, 0.04, 0.08, 0.16},
]

\addplot [
    very thick,
    domain = 0.01:0.16,
    color = blue,
    mark = *,
    error bars/.cd, 
    y fixed,
    y dir=both, 
    y explicit    
]table[x=x, y=y,y error=error, col sep=comma] {
x, y, error
0.01, 0.5207, 0.05731010410997559
0.02, 0.5937, 0.054314879613983236
0.04, 0.6407, 0.04339884842941428
0.08, 0.6817, 0.0272582040743672
0.16, 0.7085, 0.020946867900917068
};
\addlegendentry{\brc}

\addplot [
    very thick,
    dotted,
    domain=0.01:0.16, 
    samples=100, 
    color=green,
]
{0.8298};
\addlegendentry{Non-Private Baseline}

\addplot [
    very thick,
    dotted,
    domain = 0.01:0.16,
    color =red,
    mark = square,
    error bars/.cd, 
    y fixed,
    y dir=both, 
    y explicit    
]table[x=x, y=y,y error=error, col sep=comma] {
x, y, error
0.01, 0.5060, 0.08534888979110947
0.02, 0.5236, 0.0712347173281054
0.04, 0.5651, 0.06813335663141522
0.08, 0.6252, 0.0578112860677855
0.16, 0.6757, 0.044356074618669486
};
\addlegendentry{Private Baseline}

\end{axis}
\end{tikzpicture}
    \subcaption{\label{fig10}Credit Cards Dataset}
    \end{subfigure} 
        \hfill        
\begin{subfigure}[t]{0.24\linewidth}
    \centering   \begin{tikzpicture}[yscale = 0.34, xscale = 0.34][!h]
\begin{axis}[
    legend style={at={(0.55,0.25)}},
    ymin = 0, ymax = 1,
    very thick,
    axis lines = left,
    xmode=log,
    xlabel = \(\epsilon\) {(\(T = 25, c_1 = c_2 = \sqrt{2}\))},
    ylabel = {\(\text{Accuracy}\)},    
    xtick = {0.01, 0.02, 0.04, 0.08, 0.16},
    xticklabels = {0.01, 0.02, 0.04, 0.08, 0.16},
]

\addplot [
    very thick,
    domain = 0.01:0.16,
    color = blue,
    mark = *,
    error bars/.cd, 
    y fixed,
    y dir=both, 
    y explicit    
]table[x=x, y=y,y error=error, col sep=comma] {
x, y, error
0.01, 0.5251, 0.06902746502980635
0.02, 0.5882, 0.05040098704503392
0.04, 0.6561, 0.016387924965277884
0.08, 0.6742, 0.01001380324322865
0.16, 0.6803, 0.008541157174700291
};
\addlegendentry{\brc}

\addplot [
    very thick,
    dotted,
    domain=0.01:0.16, 
    samples=100, 
    color=green,
]
{0.6983};
\addlegendentry{Non-Private Baseline}

\addplot [
    very thick,
    dotted,
    domain = 0.01:0.16,
    color =red,
    mark = square,
    error bars/.cd, 
    y fixed,
    y dir=both, 
    y explicit    
]table[x=x, y=y,y error=error, col sep=comma] {
x, y, error
0.01, 0.5329, 0.09134767648124437
0.02, 0.5694, 0.08655640919512429
0.04, 0.6190, 0.08288973949316225
0.08, 0.6453, 0.04710192355013991
0.16, 0.6829, 0.009331891362941096
};
\addlegendentry{Private Baseline}

\end{axis}
\end{tikzpicture}
    \subcaption{\label{fig11}Sepsis Survival Dataset}
    \end{subfigure}    
    \hfill
\begin{subfigure}[t]{0.25\linewidth}
    \centering   \begin{tikzpicture}[yscale = 0.34, xscale = 0.34][!h]
\begin{axis}[
    legend style={at={(0.55,0.25)}},
    ymin = 0, ymax = 1,
    very thick,
    axis lines = left,
    xmode=log,
    xlabel = \(\epsilon\) {(\(T = 25, c_1 = c_2 = \sqrt{2}\))},
    ylabel = {\(\text{Accuracy}\)},    
    xtick = {0.01, 0.02, 0.04, 0.08, 0.16},
    xticklabels = {0.01, 0.02, 0.04, 0.08, 0.16},
]

\addplot [
    very thick,
    domain = 0.01:0.16,
    color = blue,
    mark = *,
    error bars/.cd, 
    y fixed,
    y dir=both, 
    y explicit    
]table[x=x, y=y,y error=error, col sep=comma] {
x, y, error
0.01, 0.5664, 0.07975775704929712
0.02, 0.7164, 0.05860303978851877
0.04, 0.7537, 0.04721444025990352
0.08, 0.7951, 0.028331590518383912
0.16, 0.8323, 0.022791374624312005
};
\addlegendentry{\brc}

\addplot [
    very thick,
    dotted,
    domain=0.01:0.16, 
    samples=100, 
    color=green,
]
{0.9688};
\addlegendentry{Non-Private Baseline}

\addplot [
    very thick,
    dotted,
    domain = 0.01:0.16,
    color =red,
    mark = square,
    error bars/.cd, 
    y fixed,
    y dir=both, 
    y explicit    
]table[x=x, y=y,y error=error, col sep=comma] {
x, y, error
0.01, 0.5075, 0.13482951578668553
0.02, 0.5810, 0.0991639153945424
0.04, 0.6408, 0.07388057513861102
0.08, 0.6946, 0.054667693309617196
0.16, 0.7687, 0.03493426519455197
};
\addlegendentry{Private Baseline}

\end{axis}
\end{tikzpicture}
    \subcaption{\label{fig12}Bank Marketing Dataset}
    \end{subfigure}        
           
    \caption{ \label{fig_alterante}Accuracy of Algorithm~\ref{algo2} for differentially private linear classification compared to baselines.}
\end{figure*}

\section{Additional Results and Discussion}
\label{sec:discuss}
In this section, we provide some additional discussion to complement our main results. We first compare \brc with previous private boosting approaches, and next compare \brc with an extension of PATE to our setting. We finally study the rate of convergence of \brc compared to a non-private baseline.

\noindent {\bf Previous Boosting Approaches.} First, we empirically compare \brc with the {\em Boosting for People} algorithm introduced in~\cite{boosting-dp}. As mentioned before, this algorithm also modifies the AdaBoost algorithm and provides bounds on both privacy and accuracy. Though this is theoretically sound, we believe the theoretical guarantees require $\epsilon$ to be unnaturally large. As mentioned before, the key reason for this is that the algorithm needs to add noise to {\em both} the weights and the accuracy computation. Under reasonable parameter setting for this algorithm, we present the comparison of its accuracy with \brc on the Adult dataset in Figure~\ref{fig5}, using the same public-private feature split as in Section~\ref{sec:experiments}. For the range of privacy budget $\epsilon$ in our experiments, the Boosting for People algorithm is not much better than randomly guessing labels, while the \brc algorithm achieves far better privacy-accuracy trade-off. 

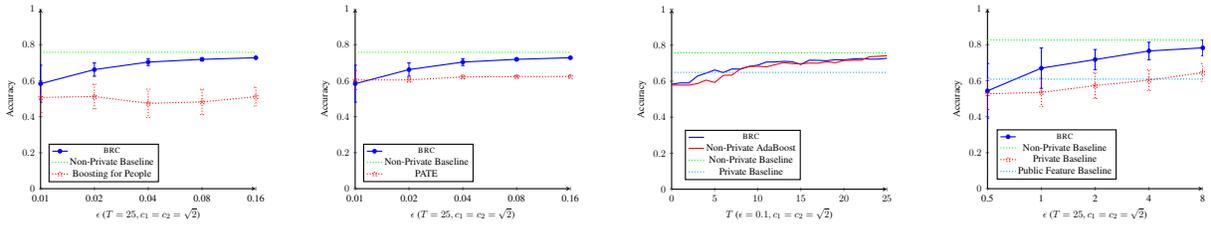
\begin{figure*}[htbp]
\centering      
\begin{subfigure}[t]{0.24\linewidth}
    \centering   \begin{tikzpicture}[yscale = 0.34, xscale = 0.34][!h]
\begin{axis}[
    legend style={at={(0.55,0.25)}},
    ymin = 0, ymax = 1,
    very thick,
    axis lines = left,
    xmode=log,
    xlabel = \(\epsilon\) {(\(T = 25, c_1 = c_2 = \sqrt{2}\))},
    ylabel = {\(\text{Accuracy}\)},    
    xtick = {0.01, 0.02, 0.04, 0.08, 0.16},
    xticklabels = {0.01, 0.02, 0.04, 0.08, 0.16},
]

\addplot [
    very thick,
    domain = 0.01:0.16,
    color = blue,
    mark = *,
    error bars/.cd, 
    y fixed,
    y dir=both, 
    y explicit    
]table[x=x, y=y,y error=error, col sep=comma] {
x, y, error
0.01, 0.5833, 0.10276637643421352
0.02, 0.6617, 0.03696339855701698
0.04, 0.7034, 0.019225346377164696
0.08, 0.7187, 0.008686005236905735
0.16, 0.7275, 0.004534542263973636
};
\addlegendentry{\brc}

\addplot [
    very thick,
    dotted,
    domain=0.01:0.16, 
    samples=100, 
    color=green,
]
{0.7575};
\addlegendentry{Non-Private Baseline}

\addplot [
    very thick,
    dotted,
    domain = 0.01:0.16,
    color =red,
    mark = square,
    error bars/.cd, 
    y fixed,
    y dir=both, 
    y explicit    
]table[x=x, y=y,y error=error, col sep=comma] {
x, y, error
0.01, 0.5070, 0.08507249667842426
0.02, 0.5121, 0.06783748300154785
0.04, 0.4743, 0.07863462921663916
0.08, 0.4816, 0.07047883633808578
0.16, 0.5113, 0.0520125162720266
};
\addlegendentry{Boosting for People}

\end{axis}
\end{tikzpicture}
    \subcaption{\label{fig5}Comparing \brc with Boosting for People on the Adult Dataset}
    \end{subfigure}         
\hfill
\begin{subfigure}[t]{0.24\linewidth}
    \centering   \begin{tikzpicture}[yscale = 0.34, xscale = 0.34][!h]
\begin{axis}[
    legend style={at={(0.55,0.25)}},
    ymin = 0, ymax = 1,
    very thick,
    axis lines = left,
    xmode=log,
    xlabel = \(\epsilon\) {(\(T = 25, c_1 = c_2 = \sqrt{2}\))},
    ylabel = {\(\text{Accuracy}\)},    
    xtick = {0.01, 0.02, 0.04, 0.08, 0.16},
    xticklabels = {0.01, 0.02, 0.04, 0.08, 0.16},
]

\addplot [
    very thick,
    domain = 0.01:0.16,
    color = blue,
    mark = *,
    error bars/.cd, 
    y fixed,
    y dir=both, 
    y explicit    
]table[x=x, y=y,y error=error, col sep=comma] {
x, y, error
0.01, 0.5833, 0.10276637643421352
0.02, 0.6617, 0.03696339855701698
0.04, 0.7034, 0.019225346377164696
0.08, 0.7187, 0.008686005236905735
0.16, 0.7275, 0.004534542263973636
};
\addlegendentry{\brc}

\addplot [
    very thick,
    dotted,
    domain=0.01:0.16, 
    samples=100, 
    color=green,
]
{0.7575};
\addlegendentry{Non-Private Baseline}

\addplot [
    very thick,
    dotted,
    domain = 0.01:0.16,
    color =red,
    mark = square,
    error bars/.cd, 
    y fixed,
    y dir=both, 
    y explicit    
]table[x=x, y=y,y error=error, col sep=comma] {
x, y, error
0.01, 0.6042, 0.05342171085028539
0.02, 0.6042, 0.01118622483372293
0.04, 0.6209, 0.008982035928143844
0.08, 0.6232, 0.005560307955517591
0.16, 0.6232, 0.007157057540924702
};
\addlegendentry{PATE}

\end{axis}
\end{tikzpicture}
    \subcaption{\label{fig60}Comparing \brc with PATE on the Adult Dataset}
    \end{subfigure}       
    \hfill
\begin{subfigure}[t]{0.24\linewidth}
    \centering   \begin{tikzpicture}[yscale = 0.34, xscale = 0.34][!h]
\begin{axis}[
    mark size = 0pt,
    very thick,
    legend style={at={(0.6,0.33)}},
    ymin = 0, ymax = 1,
    axis lines = left,
    xlabel = {$T$ (\(\epsilon = 0.1, c_1 = c_2 = \sqrt{2}\))},
    ylabel = {\(\text{Accuracy}\)},    
]
\addplot table {
0 0.5832783401348931
1 0.5902480752780154
2 0.5902480752780154
3 0.6287425149700598
4 0.6447733105218136
5 0.6625320786997434
6 0.6479897348160821
7 0.6680923866552609
8 0.6672369546621043
9 0.6817792985457656
10 0.6890504704875963
11 0.7065868263473054
12 0.7070145423438836
13 0.7104362703165098
14 0.7082976903336184
15 0.6928999144568007
16 0.716852010265184
17 0.7159965782720273
18 0.713857998289136
19 0.7194183062446535
20 0.718562874251497
21 0.723267750213858
22 0.7258340461933276
23 0.7224123182207014
24 0.7228400342172797
25 0.7275449101796407
};
\addlegendentry{\brc}

\addplot table {
0 0.5782720273738238
1 0.5782720273738238
2 0.5782720273738238
3 0.5868263473053892
4 0.6060735671514115
5 0.592814371257485
6 0.6325919589392643
7 0.6338751069289992
8 0.6689478186484175
9 0.6817792985457656
10 0.6830624465355004
11 0.6796407185628742
12 0.6916167664670658
13 0.7031650983746792
14 0.69803250641574
15 0.6971770744225834
16 0.7005988023952096
17 0.7001710863986313
18 0.7078699743370402
19 0.7027373823781009
20 0.7151411462788708
21 0.7164242942686057
22 0.7164242942686057
23 0.7360992301112061
24 0.7382378100940975
25 0.7416595380667237
};
\addlegendentry{Non-Private AdaBoost}

\addplot [
    very thick,
    dotted,
    domain=0:25, 
    samples=100, 
    color=green,
]
{0.7575};
\addlegendentry{Non-Private Baseline}

\addplot [
    very thick,
    dotted,
    domain=0:25, 
    samples=100, 
    color=cyan,
]
{0.6490};
\addlegendentry{Private Baseline}
\end{axis}
\end{tikzpicture}
    \subcaption{\label{fig6}Comparing the Convergence of \brc with non-private AdaBoost on the Adult Dataset}
    \end{subfigure}      
    \hfill
\begin{subfigure}[t]{0.24\linewidth}
    \centering   \begin{tikzpicture}[yscale = 0.34, xscale = 0.34][!h]
\begin{axis}[
    legend style={at={(0.6,0.33)}},
    ymin = 0, ymax = 1,
    very thick,
    axis lines = left,
    xmode=log,
    xlabel = \(\epsilon\) {(\(T = 25, c_1 = c_2 = \sqrt{2}\))},
    ylabel = {\(\text{Accuracy}\)},    
    xtick = {0.5, 1, 2, 4, 8},
    xticklabels = {0.5, 1, 2, 4, 8},
]
\addplot [
    very thick,
    domain = 0.5:8,
    color = blue,
    mark = *,
    error bars/.cd, 
    y fixed,
    y dir=both, 
    y explicit    
]table[x=x, y=y,y error=error, col sep=comma] {
x, y, error
0.5, 0.5435, 0.15186303649193894
1, 0.6696, 0.11203564110195761
2, 0.7174, 0.05584883729854403
4, 0.7652, 0.04841534228547847
8, 0.7826, 0.04347826086956523
};
\addlegendentry{\brc}

\addplot [
    very thick,
    dotted,
    domain=0.5:8, 
    samples=100, 
    color=green,
]
{0.8261};
\addlegendentry{Non-Private Baseline}

\addplot [
    very thick,
    dotted,
    domain = 5:8,
    color =red,
    mark = square,
    error bars/.cd, 
    y fixed,
    y dir=both, 
    y explicit    
]table[x=x, y=y,y error=error, col sep=comma] {
x, y, error
0.5, 0.5273, 0.08744286966473112
1, 0.5346, 0.07864229826311321
2, 0.5729, 0.07250240515847191
4, 0.6029, 0.05848521949151601
8, 0.6453, 0.04970063710992581
};
\addlegendentry{Private Baseline}

\addplot [
    very thick,
    dotted,
    domain=0.5:8, 
    samples=100, 
    color=cyan,
]
{0.6087};
\addlegendentry{Public Feature Baseline}

\end{axis}
\end{tikzpicture}
    \subcaption{\label{fig77}Comparing the Convergence of \brc with non-private AdaBoost on small dataset}
    \end{subfigure}  
    \caption{Empirical plots related to aspects of the discussion.}
\end{figure*}

\noindent {\bf Comparison with PATE.} We compare \brc with an extension of PATE~\cite{pate1, pate2, tran2022sfpate} to our setting. The standard usage of PATE is where part of the data is public and part of the data is private. A student model is trained on public data and an ensemble of teacher models is trained on private data. The student model queries the teacher models for labels. We consider the following extension of PATE to fit into our setting: A student model has access to public features, and an ensemble of teacher models is trained on private features. For each data point, each teacher offers a prediction on the label. We aggregate the votes of the teachers in a differentially private manner, and input the winning label as an additional feature to the students. The student model is trained and tested on the public features augmented with this winning label. Under reasonable parameter setting for this algorithm, We present the comparison of its accuracy with \brc on the Adult dataset in Figure~\ref{fig60}, using the same public-private feature split as in Section~\ref{sec:experiments}. We observe that the accuracy of PATE increases very slowly with $\epsilon$, and is superseded by \brc for $\epsilon \geq 0.02$.

\noindent {\bf Rate of Convergence of {\sc brc}.} We also empirically show that \brc converges reasonably fast compared with the non-private version of AdaBoost (Algorithm~\ref{algo0}) that trains a random classifier in each iteration. In particular, we compare the pattern of convergence of \brc with non-private AdaBoost, {\em i.e.} Algorithm~\ref{algo0} that generates a random classifier each iteration, on the Adult dataset. For \brc, we take $\epsilon = 0.1$ and $c_1 = c_2 = \sqrt{2}$. We run $25$ iterations for both algorithms. We present the results in Figure~\ref{fig6}, which also includes a private baseline that plots the accuracy of differentially private logistic regression with a privacy budget of $\epsilon = 0.01$. \brc converges in almost the same rate as non-private AdaBoost, which shows that \brc is converging reasonably fast.

\noindent {\bf Performance on Small Dataset.} We complement our result by showing that \brc performs similarly well on small datasets. We take the first 500 data points of the Adult dataset and use the same public-private feature split as in Section~\ref{sec:experiments}. We experiment with $\epsilon = 0.5, 1, 2, 4$ and $8$. For \brc, we still set $T = 25$ and $c_1 = c_2 = \sqrt{2}$. We compare \brc with the three baselines as in Section~\ref{sec:experiments}, and the results are presented in Figure~\ref{fig77}. The non-private baseline is 82.61\% and the public feature baseline is 60.87\%. The \brc algorithm consistently achieves higher accuracy than the private baseline, and exceeds the public feature baseline for $\epsilon \geq 0.02$. Moreover, it reaches roughly 79\% accuracy for $\epsilon = 8$, almost recovering the non-private baseline.

\noindent {\bf Limitation when $\epsilon$ is large.} A limitation of \brc is that its accuracy is somewhat worse than the private baseline, {\em i.e.} simply running differentially private logistic regression, when $\epsilon$ is larger, say $\epsilon = 1$. Nevertheless, \brc provides an effective way to achieve significantly better accuracy than the private baseline with smaller $\epsilon$, both in the setting with public features, and in the setting where the features and label are private.

\section{Conclusion}
We presented a simple extension to AdaBoost to handle linear classification when only a few features are private. We also showed that 
the same framework extends to standard differentially private linear classification (with all features and label private), providing a new algorithm with a simple proof of privacy and with comparable or higher accuracy than differentially private logistic regression on real datasets.

The first open question is to develop a systematic understanding of the phenomenon of boosting random linear classifiers. Though this is a robust empirical observation, it would be good to formally argue its accuracy. 
Next, \brc always produces linear classifiers. For non-linear classification with partially private features, we may need an approach not based on boosting, and we leave this as an interesting open question.


Finally, Algorithm~\ref{algo1} implicitly assumes the label is public knowledge, since in Step 4 of Algorithm~\ref{algo1}, we learn a public classifier without noise. Our framework could be extended when the labels are private by adding noise to that classifier, and coming up with a method to do this that has good empirical effectiveness is an interesting open question. We note that treating labels as private has  recently been considered as {\em label differential privacy}~\cite{esfandiari2021label, label}.


\newpage

\bibliographystyle{plain}
\bibliography{refs.bib}



\end{document}